\documentclass{article}

\usepackage{arxiv}

\usepackage[utf8]{inputenc} 
\usepackage[T1]{fontenc}    
\usepackage{hyperref}       
\usepackage{url}            
\usepackage{booktabs}       
\usepackage{amsfonts}       
\usepackage{nicefrac}       
\usepackage{microtype}      
\usepackage{lipsum}
\usepackage{graphicx}
\graphicspath{ {./images/} }

\usepackage{natbib}
\usepackage{algorithm}
\usepackage{algorithmic}
\usepackage{amsmath}
\usepackage{caption}
\usepackage{wrapfig}
\usepackage{appendix}
\usepackage{amsthm}
\newtheorem{prop}{Proposition}
\usepackage{subcaption}
\usepackage{multirow}

\title{\textsc{ReCurveflow}: A Flow Matching Framework that Learns Curved Reaction Trajectories to Predict Transition State Geometries}

\author{
Seungheun Baek \\
Department of Computer Science\\
Korea University\\
Seoul, South Korea\\
\texttt{sheunbaek@korea.ac.kr} \\
\And
Mogan Gim\thanks{Corresponding author} \\
Department of Biomedical Engineering\\
Hankuk University of Foreign Studies\\
Yongin, South Korea\\
\texttt{gimmogan@hufs.ac.kr} \\
\And
Jaewoo Kang\footnotemark[1] \\
Department of Computer Science\\
Korea University\\
Seoul, South Korea\\
\texttt{kangj@korea.ac.kr}
}

\begin{document}
\maketitle

\def\modelname{\textsc{ReCurveflow}}

\begin{abstract}
Predicting transition states (TS) in chemical reactions is crucial, as they provide insights into reaction mechanisms. Recent work on TS prediction have focused on flow matching supervised on straight linear paths that do not align with actual reaction trajectories. We propose a novel flow matching-based framework \modelname~that learns to predict TS geometries supervised on continuously curved reference paths interpolated from a full NEB-derived band of molecular geometries. We also introduce off-path correction, which grants \modelname~with the ability to produce corrective velocity fields when engaged off-path geometry states during inference rollout, leading to better resistance against exposure bias and accuracy in TS prediction. Across three data splits and six evaluation metrics, \modelname~achieves the best result on the majority of split-metric combinations against seven baselines. Qualitative analyses further show that \modelname~generates reaction trajectories with energy profiles that closely track the reference NEB path, provides initializations that ease the NEB optimization bottleneck, and exhibits the intended corrective behavior in its learned velocity fields. The \modelname~codebase is publicly available at https://github.com/dmis-lab/ReCurveflow.
\end{abstract}



\section{Introduction}

Chemical reactions are continuous configurational changes of atom-wise molecular geometry, transforming reactants into products. Based on this perspective, a reaction can be viewed as a trajectory across a potential energy surface (PES), connecting the reactant and product states (RS and PS) through a transition state (TS). A TS is a first-order saddle point of the PES. It is the maximum-energy configuration along the minimum-energy path connecting RS and PS. Because its energy determines the activation barrier of a reaction, the TS plays a central role in determining reaction feasibility (e.g., reaction rate and yield) and elucidating the underlying mechanisms of chemical reactions~\cite{eyring1935activated}. In fact, understanding TSs can benefit various applications such as catalyst design, retrosynthesis-aware drug design and the study of biological pathways~\cite{schramm2013transition}.

Weighted by the importance of TSs, computational methods have been developed to locate minimum energy paths (MEP) and characterize the configurational changes along a chemical reaction~\cite{schlegel2011geometry}. One of the most notable methods is the nudged elastic band (NEB) approach which first generates an initial trajectory of geometry "images" interpolated between the two endpoints RS and PS~\cite{jonsson1998neb, henkelman2000cineb}. These images are then optimized based on energy calculations constrained by spring forces between adjacent images. The highest energy image residing at the saddle point of the reaction trajectory is approximated as the TS. 

The Transition1X dataset has motivated researchers to develop data-driven TS geometry prediction models. Recent works employ deep generative frameworks, particularly flow matching, in accurately predicting TS geometries conditioned on those of RS and PS~\cite{duan2025reactot, galustian2025goflow}. In this downstream task, conventional flow matching defines its training objective as learning a time-dependent constant velocity field that transports a randomly sampled initial geometry to the ground-truth TS geometry along a straight linear interpolation path~\cite{lipman2022flow}.

In the Transition1X dataset~\citep{schreiner2022transition1x}, each reaction data instance consists not only the endpoint and saddle-point images (RS, PS, TS) but also intermediary images as well. We hypothesize utilizing all the images from that dataset, by means of constructing a reference path may enable velocity field supervision towards resemblance to MEP curvature. Accordingly, we exploit them to reformulate the flow matching problem. Rather than relying on a straight linear path, we first utilize cubic spline interpolation to build a \textit{continuous curved reference path} that passes through all of these discrete geometry images for each reaction. By this modification, we expect the model to generate not only accurate TS geometries but also reaction trajectories that resemble the NEB-derived MEP.

However, this reformulation may pose risks of exposure bias, partially addressed by literature on flow matching models~\cite{qin2026soar, huang2026exposure}. To alleviate this issue, we augment our reformulated flow matching with off-path correction. Given a geometry state perturbed off the reference path at a sampled timepoint, the model is additionally supervised to produce a vector field component directed back towards the reference path. By this augmented supervision, the model is equipped with off-path correction mechanism that improves resistance to exposure bias and accuracy in TS generation at inference.

The main contributions of this paper are as follows:
\begin{itemize}
    \item We reformulate flow matching for TS prediction around a continuous curved reference path spline-fit to the full NEB image band, so that \modelname~generates not only TS geometries but also reaction trajectories.
    \item We counter the exposure bias this introduces with dual off-path correction, pulling both perturbed states and the model's own rollout states back toward the reference path during model training.
    \item \modelname~improves TS accuracy over prior baselines on all three splits, generates paths with physically plausible energy profiles, and provides NEB initializations that recover the reference reaction channel more often and at lower cost.
\end{itemize}


\section{Related Works}

\subsection{Learning-based Transition State Prediction}
Recent learning-based approaches directly predict transition state (TS) geometries from reactants and products without iterative quantum optimization. TSDiff first formulated TS generation as a diffusion process conditioned on 2D molecular graphs~\cite{kim2023tsdiff}. OA-ReactDiff incorporated reactant and product 3D structures through an object-aware equivariant diffusion model~\citep{duan2023oareactdiff}. More recently, React-OT and GoFlow employed optimal transport and flow matching, respectively, for more efficient TS generation~\citep{duan2025reactot,galustian2025goflow}, and FragmentFlow further improved scalability to larger molecular systems through fragment-based generation~\citep{shprints2026fragmentflow}. None of these works exploit the intermediary waypoint images available for each reaction in the Transition1x dataset~\citep{schreiner2022transition1x}, nor do they incorporate the notion of a reaction trajectory in their model design.

\subsection{Learning-based Reaction Trajectory Generation}
Recent studies have extended transition state prediction to modeling the entire reaction pathway. MEPIN learns continuous minimum energy paths between reactants and products, while MolGen jointly generates reaction pathways, transition states, and products within a unified framework~\citep{zhang2025mepin,zhou2025molgen}. Although these methods model reaction dynamics beyond isolated transition states, they treat reaction paths as generation targets. 
We remark that these works are key motivations for constructing continuous curved reference paths from a ordered sequence of geometry images residing on NEB-derived MEP and utilizing them as supervision within our flow matching framework.

\section{Methods}

\subsection{Problem Formulation}
In a chemical reaction, let $R,T,P \in \mathbb{R}^{N\times3}$ denote the 3D atomic configurations (molecular geometries) for the reactant, transition, product states (RS, TS, PS) respectively, where $N$ is the total number of atoms. We assume that all states share the same atom ordering and atom types, such that atom-wise correspondence is known.

Given the true molecular geometries of a reactant and product $(R,P)$, our goal is to predict the geometry of its corresponding transition state $T$. We treat this as a conditional generative modeling problem by learning a velocity field $v_\theta(x,t\,|\,R,P)$, where $x\in\mathbb{R}^{N\times3}$ is a molecular geometry and $t\in[0,1]$ is flow time. 

The velocity field defines the ordinary differential equation (ODE):
\begin{equation}
\frac{dx(t)}{dt}=v_\theta(x,t\,|\,R,P),
\end{equation}
Integrating the ODE with a fixed small step size from the initial condition $x(0)=R$ yields an ordered sequence of geometries. We denote this sequence as the \textbf{generated reaction trajectory}. Here, we reparameterize flow time such that the TS lies in the trajectory midpoint ($t=0.5$), and treat $x(0.5)$ as predicted TS geometry.

\subsection{Data Preparation}
\subsubsection{Transition State Dataset.}
We use Transition1x~\cite{schreiner2022transition1x} to train and evaluate our TS geometry prediction model. This dataset comprises 10,073 organic reactions, each providing RS, TS, and PS geometries obtained via the Nudged Elastic Band (NEB) algorithm with Density Functional Theory (DFT) energy and force calculations. Beyond this triplet, every reaction is additionally annotated with an ordered sequence of intermediary geometry images which anchor the minimum-energy path (MEP) connecting the reactant and product endpoints. For clarity, we denote these atomic configurations as \textbf{On-Trajectory States (OS)}. Different from previous TS prediction methods, we propose to exploit these OS geometries in our flow-matching framework which will be elaborated in the next paragraph.

\begin{figure}[t]
\centering
\begin{minipage}[t]{0.48\textwidth}
  \vspace{0pt}
  \centering
  \includegraphics[width=0.9\linewidth]{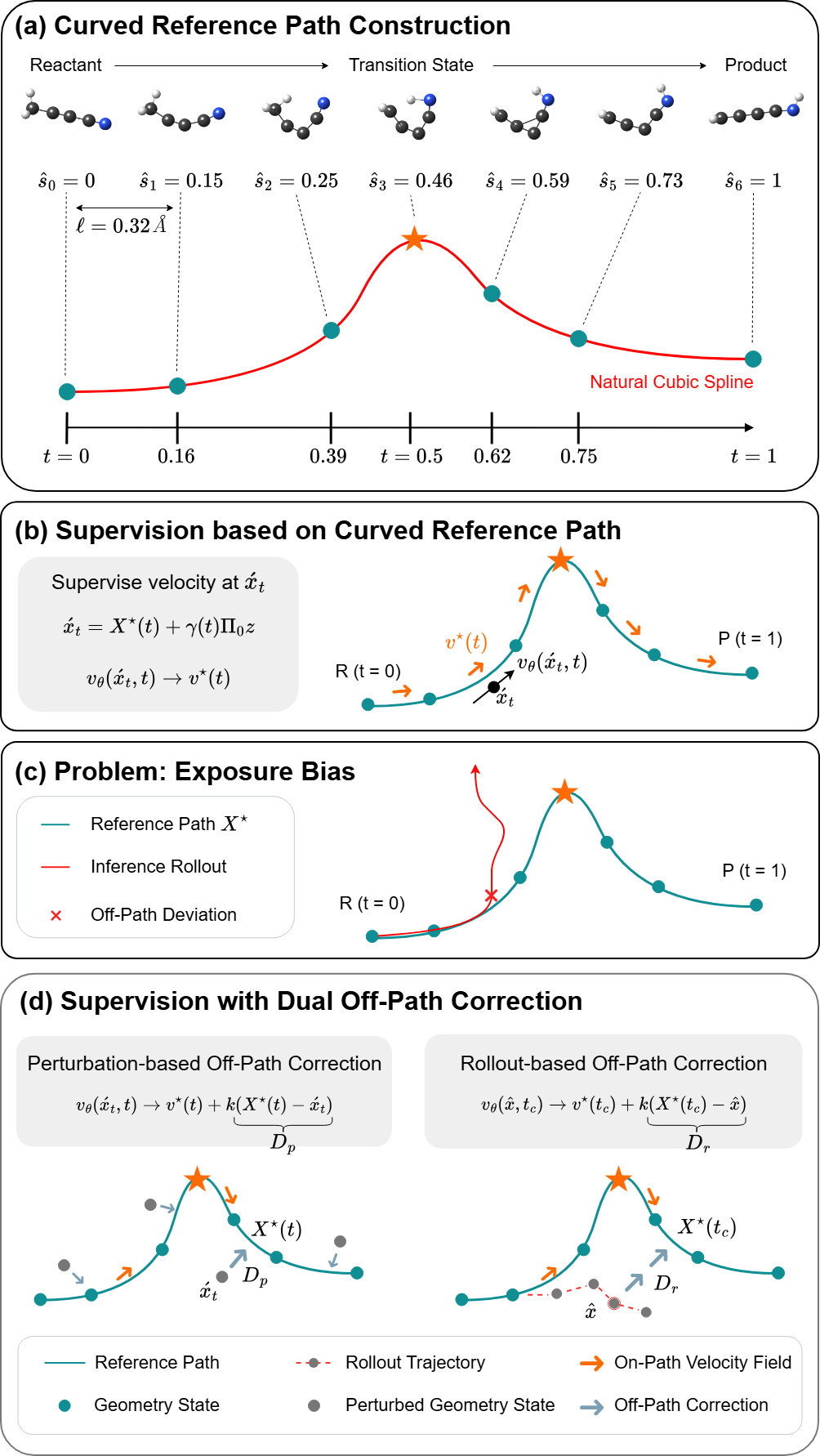}
  \caption{Overview of our framework.}
  \label{fig:main}
\end{minipage}\hfill
\begin{minipage}[t]{0.48\textwidth}
  \vspace{0pt}
  \input{algorithms/algo_1}
  \input{algorithms/algo_2}
\end{minipage}
\end{figure}

\subsubsection{Curved Reference Path Construction.}
Existing flow matching frameworks supervise the velocity field based on straight linear interpolant $x_t$, connecting a simple Gaussian prior ($x_0 \sim \mathcal{N}(0,I)$) or an initial guess based on double-ended points ($x_0 = \frac{R+P}{2}$)~\citep{duan2025reactot} to the TS ($x_1 = T$), as expressed below.
\begin{equation}
x_t = (1-t)x_0 + tx_1, \quad v_t = x_1 - x_0
\end{equation}
where $v_t$ is the ground truth velocity field derived from straight linear interpolation.

We propose to replace the linear interpolant with one that bears a resemblance to curvature of the NEB-derived MEP. That is, we intend to obtain time-varying ground truth velocity fields from a \textbf{curved reference path} where $(x_0,x_1)=(R,P)$. This requires path parameterization since the velocity field should be obtained for $t\in[0,1]$ during training time.

For each reaction, the dataset provides a \textit{discrete} ordered sequence of ten geometry images $\{X_k\}_{k=0}^{9}$ including double-ended points RS ($X_0=R$) and PS $(X_{9}=P)$, TS positioned somewhere between them ($X_m=T,\quad 0<m<9$) and the remaining OS images. We utilize these geometry images to construct a \textit{continuous} curved reference path of interpolated geometry states $X^{\star}$ used for velocity field supervision. Note that each of the OS images $X\in\mathbb{R}^{N\times3}$ share the same atom ordering and atom types with the TS.

Since these NEB-derived geometry images are unevenly spaced across the MEP, we need to adjust their atomic configurations and align their discrete positional indices ($k$) with the time indices ($t$) of our proposed interpolant. First, we reposition the geometries by centering and aligning using the Kabsch algorithm. Then, we build normalized distances across the flow time axis of the reference path:
\begin{equation}
\hat{s}_j = \frac {\sum_{i<j}\ell_i} {\sum_i\ell_i},
\end{equation}
where  $\ell_j=(\tfrac{1}{N}\|X_{j+1}-X_j\|_F^2)^{1/2}$ is the per-atom RMSD between consecutive geometries.

Subsequently, we rescale the coordinates so that the TS is coerced to be positioned at $t=0.5$:
\begin{equation}
t_j=
\begin{cases}
\dfrac{\hat{s}_j}{2\hat{s}_m},
&
\hat{s}_j\le\hat{s}_m,\\[0.8ex]
\dfrac12+
\dfrac{\hat{s}_j-\hat{s}_m}
{2(1-\hat{s}_m)},
&
\hat{s}_j>\hat{s}_m.
\end{cases}
\end{equation}
where $t_0=0$, $t_m=0.5$, $t_9$=1. Each geometry image $X_j$ then is paired with its corresponding time index $t_j$, per reaction in the dataset. 

Finally, we fit a natural cubic spline to the resulting time-geometry pairs $\{(t_j, X_j)\}_{j=0}^{9}$, yielding a $C^2$-continuous curved reference path $X^\star$ defined over $t\in[0,1]$. Here, we denote $X^{\star}\in\mathbb{R}^{N\times 3}$ as the \textbf{reference geometry state} whose time derivative $v^{\star}_t$ is the \textbf{ground truth velocity field} expressed as below:
\begin{align}
v^{\star}_t = \frac{d X^{\star}(t)}{dt}
\end{align}
where $(X^\star(0),X^\star(0.5),X^\star(1))=(R,T,P)$. We expect $v_{\theta}$ to capture local curvatures of the interpolant built from NEB-derived geometries.

\subsection{\modelname}
\subsubsection{Supervision based on Curved Reference Path.}
Given the continuous curved reference path $X^{\star}(t)$, we train $v_\theta(x,t\,|\,R,P)$ to regress its corresponding ground truth velocity $v^\star(t)$ at $t\sim\mathcal{U}(0,1)$. To provide support in a neighborhood of the path rather than on the path alone, we perturb the reference geometery state at each sampled time using mean-free Gaussian tube. The loss objective is expressed as below: 
\begin{align}
\acute{x}_t &= X^\star(t) + \gamma(t)\Pi_0z, \quad z\sim\mathcal{N}(0,I),
\\
\mathcal{L} &= \mathbb{E} \left[ \left\| v_\theta(\acute{x}_t,t\,|\,R,P) - v^\star(t) \right\|^2 \right] \label{eq:main}
\end{align}
where $\Pi_0z$ is the mean-free Gaussian perturbation and $\acute{x}_t$ is the \textbf{perturbed reference geometry state}.

\subsubsection{Supervision with Dual Off-Path Correction.}
The training objective (Eq.~\eqref{eq:main}) leaves our flow matching framework vulnerable to exposure bias (Fig~\ref{fig:main}(c)). While $v_\theta$ is fed geometry states from inside the Gaussian tube surrounding $X^\star(t)$ during training, it is fed states produced by its own preceding predictions during inference. Suppose that at some inference step, its generated velocity points off-path, showing discrepancy with its corresponding tangent of the curved reference path. Having integrated this velocity may produce a geometry state rarely observed during training. Since numerical integration proceeds as a rollout, such off-path deviations are compounded rather than corrected. Hence, a single off-path velocity can propagate through the remaining steps, eventually yielding a reaction trajectory whose midpoint deviates from the true TS geometry.

Motivated by SOAR~\citep{qin2026soar}, we devise an off-path correction term that is added to the ground truth velocity. The core idea is to supervise $v_\theta$ with velocity fields that pull off-path geometry states back onto the reference path, in contrast to SOAR, which pulls off-path states back toward the endpoint target. Given a sampled $t$, we augment the regression target with a off-path correction term that is a displacement vector from the perturbed $\acute{x}_t$ back to its reference $X^{\star}(t)$. The resulting objective is expressed as below: 
\begin{align}
D_{\mathrm{p}} &= X^{\star}(t) - \acute{x}_t, \\
\mathcal{L}_{\mathrm{p}} &= \mathbb{E}\left[ \left\| v_\theta(\acute{x}_t,t\,|\,R,P)
  - \bigl(v^\star(t) + k\,D_{\mathrm{p}}\bigr) \right\|^2 \right] \label{eq:opc1}
\end{align}
where $k$ controls the strength of the correction. We denote this as \textbf{Perturbation-based Off-Path Correction}.

We then perform a rollout by successively integrating the predicted velocities over a sub-interval $\mathcal{T}\subset[0,1]$ of the flow time axis, with stop-gradients applied. Let $t_c$ be a timestep sampled from $\mathcal{T}$, and let $\hat{x}$ denote the geometry state the rollout reaches at $t_c$, which is further perturbed with random noise so that the model is deliberately exposed to an off-path deviation. We impose additional supervision at this step by applying the same off-path correction target as before, now anchored at the geometry the model actually arrived at rather than at a perturbed geometry state of the curved reference path as expressed below:
\begin{align}
D_{\mathrm{r}} &= X^{\star}(t_c) - \hat{x}, \\
\mathcal{L}_{\mathrm{r}} &= \mathbb{E}\left[ \left\| v_\theta(\hat{x},t_c\,|\,R,P)
  - \bigl(v^{\star}(t_c) + k\,D_{\mathrm{r}}\bigr) \right\|^2 \right] \label{eq:opc2}
\end{align}
During inference rollout, we expect $v_\theta$ to produce self-correcting velocity fields when engaging off-path geometry states. We denote this as \textbf{Rollout-based Off-Path Self-Correction}.

The resulting loss objective for \modelname~is therefore expressed as below:
\begin{align}
\mathcal{L} = \mathcal{L}_{\mathrm{p}} + \lambda\mathcal{L}_{\mathrm{r}} \label{eq:total}
\end{align}
where $\lambda$ controls the strength of the rollout-based off-path self-correction.

\subsubsection{Model Architecture Details.}
We parameterize $v_\theta$ with an $E(3)$-equivariant graph neural network (EGNN)~\citep{satorras2021n} operating on a radius graph rebuilt from the current geometry $\acute{x}_t$ at each step. Each node carries its atom-type one-hot and the time $t$, together with the invariant distances $\|\acute{x}_t-R\|$ and $\|\acute{x}_t-P\|$ that condition the field on the endpoints. The network outputs an equivariant per-atom velocity, to which we add a learned anchor term $a\,(R-\acute{x}_t)+b\,(P-\acute{x}_t)$ with per-node invariant gains $(a,b)$, the output is projected to be mean-free ($\Pi_0$), granting $v_\theta$ with translational and rotational equivariance.

\subsubsection{Bidirectionally Ensembled Learning.}
Since the reactant and product state play symmetric roles as starting points, we randomly swap $(R,P)$ and reverse the reference path with probability $p_{\mathrm{swap}}$ during training, so that $v_\theta$ learns to reach the TS from both directions. At inference we integrate both the forward (RS$\rightarrow$TS) and the swapped (PS$\rightarrow$TS) trajectories and average the two midpoint geometries, which reduces variance and improves accuracy.

\subsubsection{Training \& Inference.}
Algorithm~\ref{alg:train} and~\ref{alg:sample} describe the mini-batchwise training step and inference step for \modelname~respectively. We optimize Eq.~\eqref{eq:total} with Adam, a linear warmup followed by cosine decay, and gradient clipping, and keep an EMA of the weights for evaluation. At inference, we generate the TS geometry and full reaction trajectory by a fixed-step Euler integration of $v_\theta$ from $R$ (Algorithm~\ref{alg:sample}), treating the midpoint as the predicted TS geometry, combined with the bidirectional average above. Details related to hyperparameters are available in the Appendix.

\section{Experiments}

\begin{table*}[t]
\centering
\scriptsize
\resizebox{\textwidth}{!}{
\begin{tabular}{llcccccc}
\toprule
\textbf{Split} & \textbf{Model} &
\shortstack{\textbf{RMSD}$\downarrow$ \textbf{(\AA)}} &
\shortstack{\textbf{RMSD$\circ$}$\downarrow$ \textbf{(\AA)}} &
\shortstack{\textbf{D-MAE}$\downarrow$ \textbf{(\AA)}} &
\shortstack{\textbf{Angle MAE}$\downarrow$ \textbf{($^\circ$)}} &
\shortstack{\textbf{Dihedral MAE}$\downarrow$ \textbf{($^\circ$)}} &
\textbf{Steric Clashes}$\downarrow$ \\
\midrule

\multirow{8}{*}{\textbf{Native}}
& \modelname~(\textbf{ours})
& \textbf{0.1423{\tiny$\pm$0.0029}}
& \textbf{0.1430{\tiny$\pm$0.0025}}
& \textbf{0.0953{\tiny$\pm$0.0023}}
& \underline{3.6726{\tiny$\pm$0.1665}}
& \textbf{16.3456{\tiny$\pm$0.3813}}
& \underline{0.0515{\tiny$\pm$0.0218}} \\

& \textbf{FragmentFlow}
& 0.2947{\tiny$\pm$0.0039}
& 0.2952{\tiny$\pm$0.0037}
& 0.1250{\tiny$\pm$0.0014}
& 5.6060{\tiny$\pm$0.0817}
& \underline{19.1161{\tiny$\pm$0.4772}}
& 0.1052{\tiny$\pm$0.0029} \\

& \textbf{React-OT}
& 0.3181{\tiny$\pm$0.0208}
& 0.3195{\tiny$\pm$0.0197}
& 0.1347{\tiny$\pm$0.0090}
& 6.0807{\tiny$\pm$0.3202}
& 20.6216{\tiny$\pm$1.4188}
& 0.1261{\tiny$\pm$0.0114} \\

& \textbf{GoFlow}
& 0.3320{\tiny$\pm$0.0175}
& 0.6111{\tiny$\pm$0.0138}
& 0.1108{\tiny$\pm$0.0072}
& 3.7314{\tiny$\pm$0.3844}
& 48.2212{\tiny$\pm$1.4943}
& 0.0892{\tiny$\pm$0.0031} \\

& \textbf{MolGen}
& \underline{0.2748{\tiny$\pm$0.0298}}
& 0.6240{\tiny$\pm$0.0746}
& \underline{0.0975{\tiny$\pm$0.0113}}
& \textbf{3.4853{\tiny$\pm$0.4585}}
& 47.4763{\tiny$\pm$2.0883}
& 0.0920{\tiny$\pm$0.0161} \\

& \textbf{MEPIN}
& 0.3788{\tiny$\pm$0.0032}
& 0.3944{\tiny$\pm$0.0215}
& 0.1591{\tiny$\pm$0.0024}
& 6.4861{\tiny$\pm$0.1579}
& 23.6957{\tiny$\pm$0.1985}
& \textbf{0.0439{\tiny$\pm$0.0115}} \\

& \textbf{OAReactDiff}
& 0.2790{\tiny$\pm$0.0316}
& \underline{0.2795{\tiny$\pm$0.0316}}
& 2.3899{\tiny$\pm$3.7201}
& 9.4076{\tiny$\pm$1.2400}
& 31.3326{\tiny$\pm$2.7781}
& 0.5568{\tiny$\pm$0.0785} \\

& \textbf{TSDiff}
& 0.4110{\tiny$\pm$0.0144}
& 0.5563{\tiny$\pm$0.0116}
& 0.1879{\tiny$\pm$0.0091}
& 4.1107{\tiny$\pm$0.3031}
& 67.2785{\tiny$\pm$1.7700}
& 0.0900{\tiny$\pm$0.0089} \\

\midrule

\multirow{8}{*}{\textbf{Reaction-Core}}
& \modelname~(\textbf{ours})
& \textbf{0.1463{\tiny$\pm$0.0017}}
& \textbf{0.1466{\tiny$\pm$0.0017}}
& \textbf{0.1070{\tiny$\pm$0.0017}}
& 3.5880{\tiny$\pm$0.0967}
& \textbf{13.5312{\tiny$\pm$0.1744}}
& \underline{0.0703{\tiny$\pm$0.0031}} \\

& \textbf{FragmentFlow}
& \underline{0.2951{\tiny$\pm$0.0034}}
& \underline{0.2962{\tiny$\pm$0.0030}}
& 0.1302{\tiny$\pm$0.0017}
& 5.4019{\tiny$\pm$0.0947}
& \underline{16.3244{\tiny$\pm$0.2028}}
& 0.1301{\tiny$\pm$0.0065} \\

& \textbf{React-OT}
& 0.3249{\tiny$\pm$0.0040}
& 0.3257{\tiny$\pm$0.0040}
& 0.1461{\tiny$\pm$0.0021}
& 6.0694{\tiny$\pm$0.1630}
& 17.7312{\tiny$\pm$0.2971}
& 0.1624{\tiny$\pm$0.0185} \\

& \textbf{GoFlow}
& 0.3667{\tiny$\pm$0.0110}
& 0.5972{\tiny$\pm$0.0082}
& 0.1245{\tiny$\pm$0.0040}
& \underline{3.4696{\tiny$\pm$0.2301}}
& 43.0869{\tiny$\pm$0.7722}
& 0.0909{\tiny$\pm$0.0057} \\

& \textbf{MolGen}
& 0.2954{\tiny$\pm$0.0272}
& 0.6157{\tiny$\pm$0.0182}
& \underline{0.1095{\tiny$\pm$0.0095}}
& \textbf{3.2704{\tiny$\pm$0.2972}}
& 47.2822{\tiny$\pm$1.5692}
& 0.0886{\tiny$\pm$0.0095} \\

& \textbf{MEPIN}
& 0.3689{\tiny$\pm$0.0045}
& 0.3750{\tiny$\pm$0.0051}
& 0.1595{\tiny$\pm$0.0018}
& 5.7755{\tiny$\pm$0.1048}
& 18.8219{\tiny$\pm$0.1820}
& \textbf{0.0550{\tiny$\pm$0.0040}} \\

& \textbf{OAReactDiff}
& 0.3520{\tiny$\pm$0.0400}
& 0.3521{\tiny$\pm$0.0400}
& 6.4421{\tiny$\pm$11.9996}
& 11.9197{\tiny$\pm$2.1303}
& 35.8212{\tiny$\pm$4.5869}
& 1.0148{\tiny$\pm$0.1896} \\

& \textbf{TSDiff}
& 0.4882{\tiny$\pm$0.0103}
& 0.6231{\tiny$\pm$0.0047}
& 0.2374{\tiny$\pm$0.0085}
& 4.0037{\tiny$\pm$0.2477}
& 67.0407{\tiny$\pm$0.9258}
& 0.0823{\tiny$\pm$0.0041} \\

\midrule

\multirow{8}{*}{\textbf{Barrier}}
& \modelname~(\textbf{ours})
& \textbf{0.1747{\tiny$\pm$0.0009}}
& \textbf{0.1755{\tiny$\pm$0.0011}}
& \underline{0.1245{\tiny$\pm$0.0010}}
& \underline{4.2977{\tiny$\pm$0.0523}}
& \textbf{16.1259{\tiny$\pm$0.1606}}
& \underline{0.0707{\tiny$\pm$0.0052}} \\

& \textbf{FragmentFlow}
& 0.3328{\tiny$\pm$0.0027}
& \underline{0.3354{\tiny$\pm$0.0032}}
& 0.1442{\tiny$\pm$0.0017}
& 5.8395{\tiny$\pm$0.1459}
& \underline{17.9202{\tiny$\pm$0.1100}}
& 0.1791{\tiny$\pm$0.0063} \\

& \textbf{React-OT}
& 0.3472{\tiny$\pm$0.0089}
& 0.3479{\tiny$\pm$0.0087}
& 0.1508{\tiny$\pm$0.0036}
& 6.0959{\tiny$\pm$0.0659}
& 18.5204{\tiny$\pm$0.3830}
& 0.1762{\tiny$\pm$0.0058} \\

& \textbf{GoFlow}
& 0.9996{\tiny$\pm$0.0002}
& 0.5977{\tiny$\pm$0.0122}
& 1.3980{\tiny$\pm$0.3833}
& 56.7464{\tiny$\pm$4.3317}
& 88.4164{\tiny$\pm$1.0570}
& 19.2725{\tiny$\pm$15.4847} \\

& \textbf{MolGen}
& \underline{0.3309{\tiny$\pm$0.0574}}
& 0.6274{\tiny$\pm$0.0328}
& \textbf{0.1232{\tiny$\pm$0.0240}}
& \textbf{3.7634{\tiny$\pm$0.7697}}
& 47.0931{\tiny$\pm$2.3636}
& 0.1418{\tiny$\pm$0.0123} \\

& \textbf{MEPIN}
& 0.4033{\tiny$\pm$0.0239}
& 0.4126{\tiny$\pm$0.0084}
& 0.1731{\tiny$\pm$0.0087}
& 6.1290{\tiny$\pm$0.0557}
& 20.2533{\tiny$\pm$0.9519}
& \textbf{0.0497{\tiny$\pm$0.0588}} \\

& \textbf{OAReactDiff}
& 0.3579{\tiny$\pm$0.0125}
& 0.3580{\tiny$\pm$0.0125}
& 14.5856{\tiny$\pm$19.9701}
& 12.3039{\tiny$\pm$0.7928}
& 35.3215{\tiny$\pm$1.3716}
& 1.1102{\tiny$\pm$0.0580} \\

& \textbf{TSDiff}
& 0.4989{\tiny$\pm$0.0063}
& 0.6211{\tiny$\pm$0.0044}
& 0.2504{\tiny$\pm$0.0070}
& 4.6180{\tiny$\pm$0.2802}
& 67.1732{\tiny$\pm$0.8333}
& 0.1505{\tiny$\pm$0.0065} \\

\bottomrule
\end{tabular}
}
\caption{Transition state prediction results on different data splits. Lower is better for all metrics. The best and second-best results are highlighted in bold and underlined, respectively.}
\label{tab:main_results}
\end{table*}

\subsection{Experimental Settings}

As mentioned in the previous section, we use the Transition1x dataset~\citep{schreiner2022transition1x} in our experiments on TS geometry prediction. We trained and evaluated \modelname~and its baseline models on three types of data splits which are \textbf{Native}, \textbf{Reaction-Core} and \textbf{Barrier}. The original \textbf{Native} split provided by the dataset partitions reactions by molecular formula. \textbf{Reaction-Core} partitions reactions by atom-mapped reaction cores, measuring generalization to unseen bond rearrangements, while \textbf{Barrier} holds out reactions from the tails of the activation-energy distribution, evaluating extrapolation to unseen energy regimes. 

We evaluate the generated transition states using metrics for structural accuracy and chemical validity. Following prior work~\citep{kim2023tsdiff,duan2023oareactdiff,duan2025reactot}, our primary evaluation metric is the root-mean-square deviation (\textbf{RMSD}) between generated and reference TS geometries. We report the permutation-, rigid-alignment-, and reflection-invariant \textbf{RMSD} computed by the \texttt{pymatgen} molecule matcher~\citep{ong2013pymatgen}, which treats a structure and its mirror image as equivalent.

We additionally report its chirality-preserving variant \textbf{RMSD}$\circ$, which disallows reflection during structural alignment, following the convention applied to the original baseline implementations. To further assess local structural consistency, we additionally report the mean absolute error of the pairwise interatomic distance matrix (\textbf{D-MAE}). Finally, we evaluate chemical validity using the bond-angle and dihedral-angle MAE (\textbf{Angle MAE}, \textbf{Dihedral MAE}), as well as the number of \textbf{Steric Clashes}, defined as non-bonded atom pairs separated by less than $1.2$~\AA. We repeated all experiments with five random seeds (0--4) and report the mean and standard deviation across seeds.

\subsection{Results on Transition State Prediction}

Table~\ref{tab:main_results} summarizes the quantitative results on the three evaluation splits. Across the 18 split-metric combinations, \modelname~overall ranked first on 11 and second on 6. Notably, \modelname~outperformed its baselines with minimal discrepancy between \textbf{RMSD} and its chirality-preserving variant, whereas \textbf{MolGEN}~\citep{zhou2025molgen}, \textbf{GoFlow}~\citep{galustian2025goflow}, and \textbf{TSDiff}~\citep{kim2023tsdiff} showed gaps exceeding $0.1\,\AA$. Also, \modelname~achieved the best \textbf{Dihedral MAE} across all data splits, highlighting its ability to effectively utilize its learned 3D-equivariant features of geometry states in producing vector fields during model inference.

According to the results, \modelname~did not rank first on \textbf{Angle MAE} or \textbf{Steric Clashes}. On \textbf{Angle MAE}, \textbf{MolGEN} ranked first across all three splits. We attribute this to its edge featurization, which combines radial and angular bases and therefore encodes three-body geometry directly, whereas our EGNN backbone conditions only on interatomic distances. On \textbf{Steric Clashes}, \textbf{MEPIN} ranked first across all splits. We attribute this to its energy-based training objective where it optimizes a MaxFlux path functional while our supervision is purely geometric as $v_\theta$ is regressed onto NEB-derived reference velocities without any explicit energy-related components in our flow matching framework.

\subsection{Ablation Study on Off-Path Correction}

\begin{wraptable}{r}{0.52\textwidth}
    \centering
    \vspace{-1.0em}
    \small
    \setlength{\tabcolsep}{3.0pt}
    \renewcommand{\arraystretch}{1.10}

    \resizebox{\linewidth}{!}{
    \begin{tabular}{lccccc}
    \toprule
    \textbf{Model} &
    \textbf{RMSD} &
    \textbf{D-MAE} &
    \textbf{Angle MAE} &
    \textbf{Dihedral MAE} &
    \textbf{Steric} \\
    \midrule

    \modelname
    & \shortstack{\textbf{0.1423}\\{\scriptsize(0.0029)}}
    & \shortstack{\textbf{0.0953}\\{\scriptsize(0.0023)}}
    & \shortstack{\textbf{3.6726}\\{\scriptsize(0.1665)}}
    & \shortstack{\textbf{16.3456}\\{\scriptsize(0.3813)}}
    & \shortstack{\textbf{0.0515}\\{\scriptsize(0.0218)}} \\

    \textbf{w/o POPC}
    & \shortstack{0.1510\\{\scriptsize(0.0051)}}
    & \shortstack{0.0999\\{\scriptsize(0.0047)}}
    & \shortstack{3.8442\\{\scriptsize(0.1939)}}
    & \shortstack{17.2167\\{\scriptsize(0.7750)}}
    & \shortstack{0.0861\\{\scriptsize(0.0075)}} \\

    \textbf{w/o ROPC}
    & \shortstack{0.1538\\{\scriptsize(0.0023)}}
    & \shortstack{0.1071\\{\scriptsize(0.0021)}}
    & \shortstack{4.3021\\{\scriptsize(0.0955)}}
    & \shortstack{17.6225\\{\scriptsize(0.2269)}}
    & \shortstack{0.0924\\{\scriptsize(0.0117)}} \\

    \textbf{w/o both}
    & \shortstack{0.1680\\{\scriptsize(0.0030)}}
    & \shortstack{0.1261\\{\scriptsize(0.0035)}}
    & \shortstack{5.4284\\{\scriptsize(0.1811)}}
    & \shortstack{19.1725\\{\scriptsize(0.5580)}}
    & \shortstack{0.1042\\{\scriptsize(0.0168)}} \\

    \bottomrule
    \end{tabular}
    }

    \caption{Ablation study on Off-Path Self-Correction.}
    \label{tab:ablation}
    \vspace{-1.0em}
\end{wraptable}

We conducted an ablation study on dual off-path correction scheme, removing either or both components from the training objective of \modelname. \modelname~(\textbf{w/o POPC}) denotes the model trained without the perturbation-based off-path correction term (i.e., without $D_{\mathrm{p}}$ in Eq.~\eqref{eq:opc1}); \modelname~(\textbf{w/o ROPC}) denotes the model trained without performing the rollout stage (i.e. omitting (B) within the minibatch loop in Alg.~\ref{alg:train}); and \modelname~(\textbf{w/o both}) denotes the model trained with neither, reducing to the base objective in Eq.~\eqref{eq:main}.

As shown in Table~\ref{tab:ablation}, supervising \modelname~on the reference velocity fields alone significantly exhibits worst results all evaluation metrics, demonstrating that both components are vital for mitigating exposure bias (\textbf{w/o both}). Notably, skipping the rollout step (\textbf{w/o ROPC}) degrades performance more than removing the correction term from $\mathcal{L}_{\mathrm p}$ (\textbf{w/o POPC}). This suggests that learning to pull back from rolled-out, off-path trajectories matters more than correcting perturbed states sampled around the reference path.

\subsection{Visualization of Energy Profiles}

\begin{wrapfigure}{r}{0.43\textwidth}
    \centering
    \vspace{-1.0em}
    \includegraphics[width=0.41\textwidth]{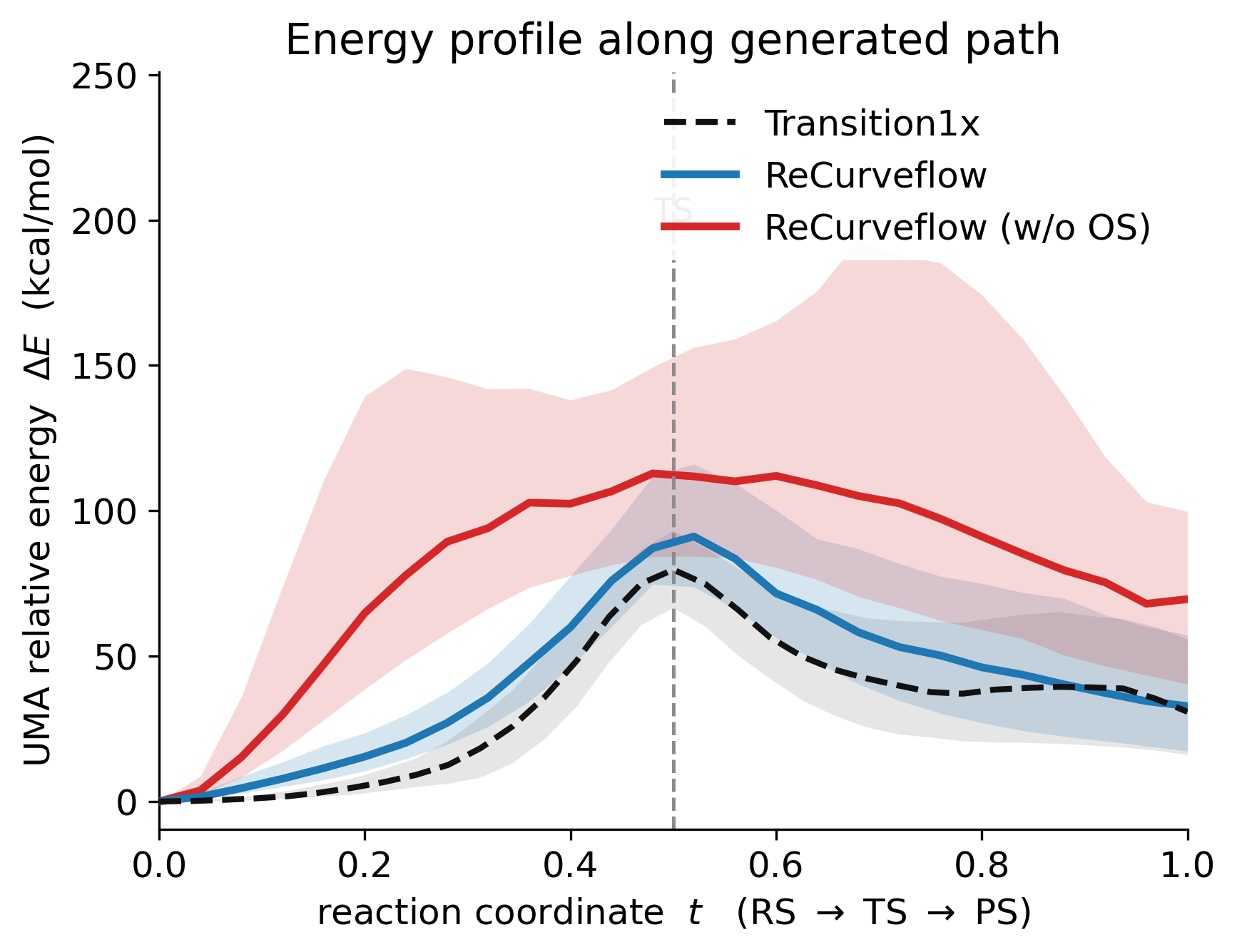}
    \caption{Comparison between \modelname~and \modelname~w/o OS based on energy profiles.}
    \label{fig:energy_profile}
    \vspace{-1.0em}
\end{wrapfigure}

As mentioned above, our reformulated flow matching framework enables \modelname~to generate reaction trajectories of atomic configurations spanning from RS to PS. For all test reactions of the \textbf{Native} split, we computed UMA energy profiles along the model-generated reaction trajectories and compared them against profiles computed the same way from ten geometry images provided by the Transition1x dataset. 

Fig.~\ref{fig:energy_profile} shows the resulting energy profiles, where the solid line and shaded region denote the mean and standard deviation across test reactions at each point along the trajectory. Here, the energy profiles derived from model-generated trajectories (\modelname) and from the reference NEB images (Transition1x) exhibit closely matching shapes, with both curves peaking near the transition state and following comparable energetic trends across reaction progress.

To further examine the benefits of using all images in constructing curved reference paths via natural cubic spline interpolation, we implemented an ablated variant by using only the PS, TS and RS images. We denote this ablation as \modelname~(\textbf{w/o OS}). Quantitative results on TS geometry prediction are available in the Appendix. As shown in Fig.~\ref{fig:energy_profile}, reaction trajectories generated by \modelname~(\textbf{w/o OS}) exhibit substantially higher variability across test reactions, with less consistent energy peaks than the other two variants. The local curvature in both the ascending (RS$\to$TS) and descending (TS$\to$PS) branches is also markedly weaker, indicating that it underfits the sharpness of the energy barrier. These results indicate that supervising the reference path on the full NEB image band helps \modelname~generate more physically plausible reaction trajectories.

\subsection{Reaction Path Initialization for NEB}

\label{sec:neb_initialization}

\begin{figure}[t]
    \centering
    \includegraphics[width=\textwidth]{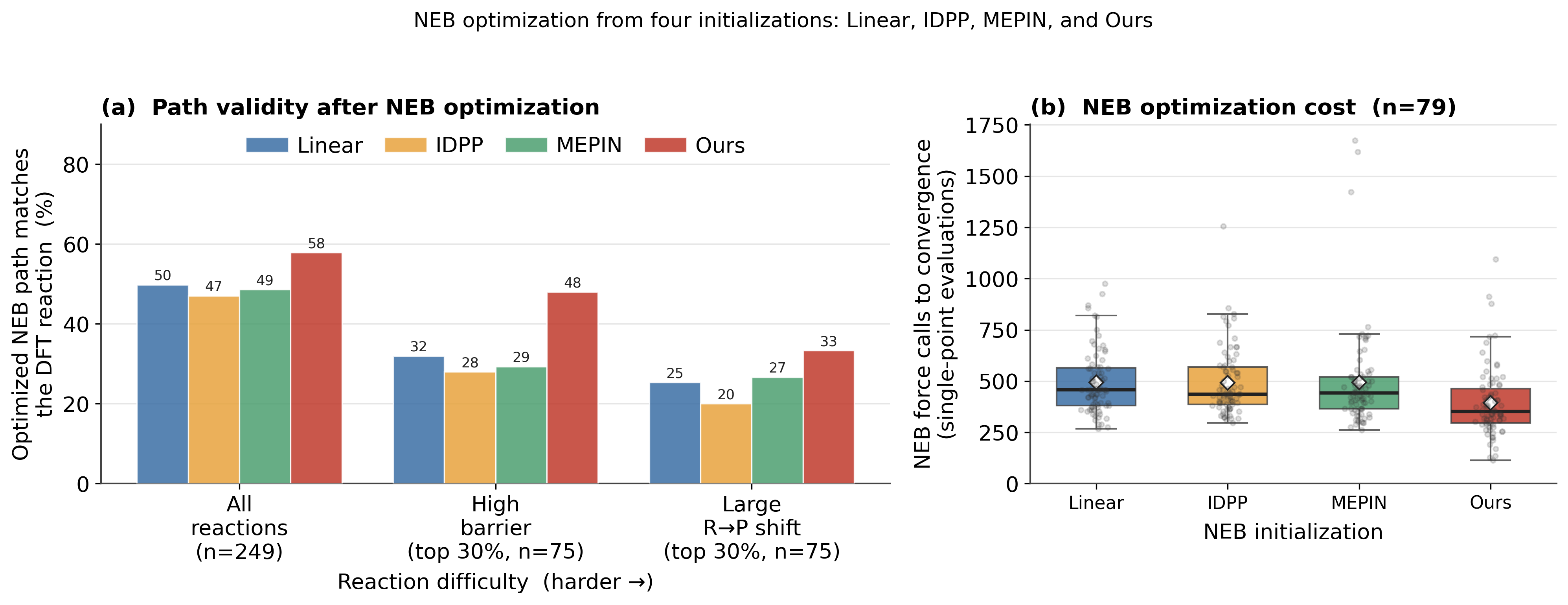}
    \caption{Comparison of NEB optimization from different initializations.}
    \label{fig:neb}

\end{figure}

The main goal of learning-based reaction path prediction methods, notably MEPIN~\citep{zhang2025mepin} is to provide feasible initial guesses of geometry images that can substantially accelerate the convergence of nudged elastic band (NEB) optimization, a well-known cost bottleneck in computational chemistry workflows~\citep{zhao2025harnessing}. We therefore evaluated \modelname~in this practical aspect, measuring how its generated reaction paths act as good initializations for quantum-chemical NEB optimization.

We compared four initialization strategies, Cartesian \textbf{Linear} Interpolation, \textbf{IDPP} interpolation~\citep{smidstrup2014improved} and using the reaction trajectories of geometry states generated by MEPIN~\citep{zhang2025mepin} and \modelname. All paths were resampled by arc length to contain the same number of images and are optimized using identical NEB settings. We first calculated the proportion of generated reaction paths with validity based on NEB optimization results (i.e convergence). As shown in Fig.~\ref{fig:neb}(a), \modelname~achieved the highest figures in all three reaction groups (All Reactions, High Barrier, Large Geometry Shift from Reactant to Product State). The results are further emphasized in challenging reaction subsets where other methods showed less than 40\% validity rate. Details of the analysis procedure are available in the Appendix.

Furthermore, we evaluated NEB optimization efficiency, which is calculated based on number of NEB force calls, on the reactions of which all four methods converged to the correct DFT saddle point. As shown in Fig.~\ref{fig:neb}(b), \modelname~delivered an approximately $1.3\times$ improvement in optimization efficiency over \textbf{Linear} and \textbf{IDPP}, requiring significantly fewer single-point evaluations to reach convergence (paired Wilcoxon signed-rank test, $p<0.001$). These results demonstrate our flow matching framework learns physically grounded configurational changes from continuous curved reference reaction paths and consequently provides reliable, initial geometry images that require smaller corrections during NEB optimization.

\subsection{Analysis on Off-Path Self-Correction Dynamics}

\begin{figure}[t]
    \includegraphics[width=\textwidth]{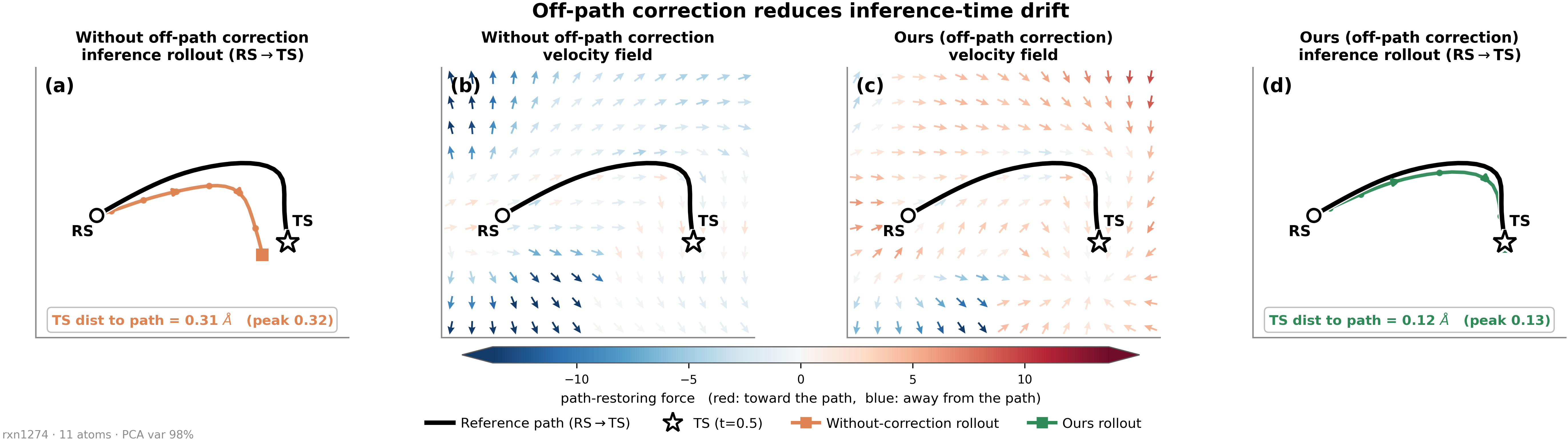}
    \caption{Analysis on Off-Path Correction Dynamics
    }
    \label{fig:exposure_bias}
\end{figure}

To verify that the proposed Supervised Correction method behaves as intended, we examined \modelname~deals with off-path geometry encounters during inference rollouts. Given the generated reaction trajectory of molecular geometries from RS to TS, we fit a two-dimensional PCA basis so that the projected components constitute the curved reference path (Figure~\ref{fig:exposure_bias}). We then overlaid a uniform square grid over this plane, covering them both on and far from the projected path, and evaluated the learned velocity at each grid point to display the local flow structure surrounding the path. Each point was reconverted to a molecular geometry state via PCA-inverse transform, and assigned the timepoint of its nearest neighbor on the reference path. Each of these projections represent a directional vector ($\langle v,(X^\star-x)/\|X^\star-x\|\rangle$), where red and blue ones indicate motion toward and away from the reference path, respectively.

Through this analysis method, we compared the velocity fields and inferred reaction trajectories produced by \modelname~with \modelname~(\textbf{w/o both}), an ablated variant supervised without dual off-path correction. According to Fig.~\ref{fig:exposure_bias}(a), the trajectory (RS$\rightarrow$TS) inferred from \modelname~trained without supervised correction sways away from the actual reference path by a distance margin of 0.31\AA{}. Its corresponding velocity field in Fig.~\ref{fig:exposure_bias}(b) shows \modelname~(\textbf{w/o both})'s velocity fields emitting weak or even negative path correction forces. In contrast, \modelname~produces consistently positive path correction forces in the same region (Fig.~\ref{fig:exposure_bias}(c)), demonstrating the effectiveness of our proposed method. Consequently, the generated trajectory in Fig.~\ref{fig:exposure_bias}(d) resembles the reference path, reducing the distance between the predicted TS and the reference path from $0.31$\,\AA{} to $0.12$\,\AA{}. Although Fig.~\ref{fig:exposure_bias} presents a single reaction (rxn1274) case, the same behavior was consistently observed across other test reactions, of which their visualizations are available in the Appendix.

\section{Conclusion}

We introduce \modelname, a flow matching framework for transition-state geometry and reaction-path prediction that replaces linear interpolants with curved reference paths derived from NEB minimum-energy paths. One limitation is data-related where existing reaction-path resources cover mostly small organic molecules over a narrow range of atomic species, which bounds how far the model can be expected to generalize to larger and more realistic chemical systems. As future work, we plan to leverage fast and accurate machine-learned energy models to perform large-scale NEB calculations on more diverse reactions, enabling larger reaction-path datasets and further improving the scalability and practical applicability of our framework.

\bibliographystyle{unsrtnat}
\bibliography{references}


\clearpage
\appendix
\renewcommand{\thesection}{\Alph{section}.\arabic{section}}
\setcounter{section}{0}

\begin{appendices}

\section{Analysis of Exposure Bias under Curved Reference Path Supervision}
\label{appendix:path_analysis}

Recent studies have shown that exposure bias in flow matching originates from the mismatch between the state distribution observed during training and that encountered during inference~\citep{qin2026soar}. During training, the velocity field is supervised only on reference states sampled from the prescribed training trajectory by minimizing

\begin{equation}
\mathcal{L}_{\mathrm{FM}}
=
\mathbb{E}_{x\sim p^{\mathrm{train}}}
\left[
\|v_\theta(x,t)-u(x,t)\|^2
\right],
\end{equation}

where $p^{\mathrm{train}}$ denotes the training-state distribution, $u(x,t)$ is the ground-truth velocity field, and $v_\theta(x,t)$ is the learned velocity field.

During inference, however, the model recursively generates its own trajectory by numerically integrating the learned velocity field,

\begin{equation}
\hat{x}_{n+1}
=
\hat{x}_n
+
h\,v_\theta(\hat{x}_n,t_n),
\end{equation}

where $\hat{x}_n$ denotes the rollout state at integration step $n$ and $h$ is the integration step size. Since each prediction depends on previous predictions, even a small velocity prediction error perturbs subsequent rollout states. Consequently, inference is performed on a model-induced state distribution $p^\theta$ rather than the training distribution,

\[
p^\theta
\neq
p^{\mathrm{train}},
\]

which constitutes the exposure bias analyzed in prior work~\citep{qin2026soar}.

The above analysis explains why rollout states gradually deviate from the reference trajectory. However, it does not characterize how such rollout deviations affect the supervision target itself when the reference trajectory is \textit{curved} instead of \textit{straight}. Since \textsc{ReCurveFlow} is supervised on continuously \textbf{curved reference paths}, obtained by fitting a natural cubic spline to the ordered sequence of geometry images for each reaction, rather than affine linear interpolants, we analyze this additional geometric effect below.

\subsection{Curvature-dependent Sensitivity induced by Reference Path Time-Shift}

Let

\[
X^\star:[0,1]\rightarrow\mathbb{R}^{N\times3}
\]

denote the curved reference path, parameterized by the flow time $t$, and let

\[
v^\star(t)
=
X^{\star\prime}(t)
\]

be its tangent velocity field.

To isolate the effect of rollout deviation, we consider the idealized case where the generated geometry remains exactly on the reference path but is shifted by a small time offset $\delta$,

\[
x
=
X^\star(t+\delta).
\]

Here, $\delta$ measures the discrepancy between the rollout state's actual progress along the reference path and its assigned flow time. Although the generated geometry still lies on the reference path, the tangent direction that locally follows the reference path is now

\[
v^\star(t+\delta),
\]

whereas the supervision target associated with the current flow time remains

\[
v^\star(t).
\]

Therefore, the supervision mismatch induced solely by the time offset is

\[
v^\star(t+\delta)-v^\star(t).
\]

\begin{prop}[Curvature-dependent Tangent Mismatch]

Assume that the curved reference path

\[
X^\star:[0,1]\rightarrow\mathbb{R}^{N\times3}
\]

is twice continuously differentiable.

Then the supervision mismatch induced by a time offset satisfies

\[
v^\star(t+\delta)-v^\star(t)
=
X^{\star\prime\prime}(t)\delta
+
o(|\delta|).
\]

Consequently,

\[
\lim_{\delta\rightarrow0}
\frac{
\|v^\star(t+\delta)-v^\star(t)\|
}
{|\delta|}
=
\|X^{\star\prime\prime}(t)\|.
\]

\end{prop}

\begin{proof}

By definition, the tangent velocity field is

\[
v^\star(t)
=
X^{\star\prime}(t).
\]

Hence,

\[
v^\star(t+\delta)-v^\star(t)
=
X^{\star\prime}(t+\delta)
-
X^{\star\prime}(t).
\]

Since $X^\star$ is assumed to be twice continuously differentiable, its tangent field varies smoothly with respect to the path parameter. Therefore, for a sufficiently small time offset $\delta$, the tangent at $t+\delta$ can be approximated by the first-order Taylor expansion around $t$,

\[
X^{\star\prime}(t+\delta)
=
X^{\star\prime}(t)
+
X^{\star\prime\prime}(t)\delta
+
o(|\delta|).
\]

Substituting this expansion into the previous equation gives

\[
\begin{aligned}
v^\star(t+\delta)-v^\star(t)
&=
X^{\star\prime}(t+\delta)
-
X^{\star\prime}(t)
\\
&=
X^{\star\prime\prime}(t)\delta
+
o(|\delta|),
\end{aligned}
\]

which proves the first result.

Finally,

\[
\frac{
\|v^\star(t+\delta)-v^\star(t)\|
}
{|\delta|}
=
\left\|
X^{\star\prime\prime}(t)
+
\frac{o(|\delta|)}{|\delta|}
\right\|,
\]

and since

\[
\frac{o(|\delta|)}{|\delta|}
\rightarrow
0
\qquad
(\delta\rightarrow0),
\]

we obtain

\[
\lim_{\delta\rightarrow0}
\frac{
\|v^\star(t+\delta)-v^\star(t)\|
}
{|\delta|}
=
\|X^{\star\prime\prime}(t)\|.
\]

\end{proof}

The proposition shows that under curved reference path supervision, even when the rollout state remains on that path, a small time shift of it changes the locally appropriate supervision target in proportion to the local curvature of the path. By contrast, an affine linear reference path satisfies

\[
X^{\star\prime\prime}(t)=0,
\]

which implies

\[
v^\star(t+\delta)=v^\star(t)
\]

for every sufficiently small $\delta$. Therefore, while the origin of exposure bias remains the train-inference state-distribution mismatch identified in prior flow-matching literature, curved reference-path supervision introduces an additional curvature-dependent sensitivity once rollout states become misaligned with their assigned flow times. This observation provides theoretical motivation for explicitly correcting rollout states back toward the reference path during training.

\section{\textsc{Recurveflow}}
\subsection{Details on \textsc{ReCurveflow}'s Hyperparameters}

\begin{table}[H]
\centering
\caption{Hyperparameters used for training ReCurveflow.}
\label{tab:hyperparameters}
\resizebox{\columnwidth}{!}{
\begin{tabular}{lll}
\toprule
Group & Hyperparameter & Value \\
\midrule
Architecture & layers / hidden dim / cutoff & 6 / 512 / 6\,\AA \\
             & max neighbors & 32 \\
Gaussian tube & $\sigma$ / $\sigma_{\mathrm{fl}}$ & 0.05 / 0.05 \\
Off-Path Correction & $k$ / $\lambda$ & 10 / 0.5 \\
Rollout & $K$ / $\beta$ / $\sigma_{\mathrm a}$ & 25 / 0.5 / 0.1 \\
Bidirectional ensemble & $p_{\mathrm{swap}}$ & 0.5 \\
Inference & Euler steps $N$ & 50 \\
Optimization & optimizer / lr / warmup & Adam / $5\times10^{-4}$ / 500 \\
             & batch / epochs / grad clip & 128 / 1500 / 1.0 \\
             & EMA decay & 0.999 \\
\bottomrule
\end{tabular}
}
\end{table}

\begin{itemize}\itemsep2pt
\item \textbf{layers / hidden dim} (6 / 512) --- depth and width of the EGNN velocity field.
\item \textbf{cutoff} (6\,\AA) --- radius for building the edges of the interaction graph.
\item \textbf{max neighbors} (32) --- per-atom cap on that graph, bounding memory for the largest systems.
\item $\boldsymbol{\sigma}$ (0.05) --- width of the Gaussian tube at its widest point, $\gamma(t)=\sigma\sqrt{t(1-t)}+\sigma_{\mathrm{fl}}$.
\item $\boldsymbol{\sigma_{\mathrm{fl}}}$ (0.05) --- constant floor on $\gamma(t)$; keeps the tube open at the endpoints, where the bridge profile vanishes and the sampler starts.
\item $\boldsymbol{k}$ (10) --- gain of the corrective term $k\,(X^\star(t)-x_t)$ added to the target velocity; $1/k=0.1$ is the flow-time constant of the pull back onto the path.
\item $\boldsymbol{\lambda}$ (0.5) --- weight of the rollout term relative to the on-path term in the total loss.
\item $\boldsymbol{K}$ (25) --- Euler steps of the stop-gradient rollout that generates off-path training states.
\item $\boldsymbol{\beta}$ (0.5) --- flow-time span covered by that rollout, so its step size is $\beta/K=0.02$.
\item $\boldsymbol{\sigma_{\mathrm a}}$ (0.1) --- amplitude of the extra perturbation on the rollout state, spreading supervision over a range of deviation magnitudes.
\item $\boldsymbol{p_{\mathrm{swap}}}$ (0.5) --- probability of exchanging the reactant and product roles during training, making one field valid in both directions.
\item \textbf{Euler steps} $\boldsymbol{N}$ (50) --- inference discretization; even, so that the transition state at $t=0.5$ falls on a step boundary.
\item \textbf{optimizer / lr / warmup} (Adam / $5\times10^{-4}$ / 500) --- linear warmup over 500 steps, then cosine decay to $1\%$ of the peak.
\item \textbf{batch / epochs / grad clip} (128 / 1500 / 1.0) --- 128 reactions per device over 4 GPUs (effective 512); gradients clipped to global norm 1.0.
\item \textbf{EMA decay} (0.999) --- exponential moving average of the weights, used for all reported evaluations.
\end{itemize}

\subsection{Details on \textsc{ReCurveflow}'s Computational Specs}
All experiments were conducted on NVIDIA RTX 3090 (24\,GB) GPUs. Our model contains $11.6$M parameters. Training requires approximately $8$--$10$ hours on $4$ GPUs, while inference takes $15.7$\,ms per reaction on a single GPU with a peak memory usage of $0.23$\,GB (batch size $128$).

\section{Visualization of Energy Profiles}

\subsection{Ablation Results related to Inclusion of OS Images in Curved Reference Path Construction}

\begin{table}[H]
\centering

\small
\renewcommand{\arraystretch}{1.1}
\setlength{\tabcolsep}{6pt}

\begin{tabular}{lcc}
\toprule
\textbf{Metric} & \textbf{Ours} & \textbf{w/o OS} \\
\midrule
Path RMSD ($\downarrow$)        & \textbf{0.101} & 0.129 \\
R/P RMSD ($\downarrow$)         & \textbf{0.013} & 0.058 \\
Barrier Overshoot ($\downarrow$)& \textbf{2.38}  & 9.26 \\
R/P Energy Error ($\downarrow$) & \textbf{0.055} & 1.90 \\
\bottomrule
\end{tabular}
\caption{Effect of On-Trajectory State (OS) images on generated reaction paths. Lower is better for all metrics.}
\label{tab:ablation_ws}
\end{table}

To quantitatively evaluate the effect of including On-Trajectory States (OS) in constructing the curved reference path, we compare the generated reaction trajectories using four path-level metrics. \textbf{Path RMSD} measures the average geometric discrepancy between the generated trajectory and the reference reaction path after arc-length alignment. \textbf{Endpoint RMSD} evaluates whether the generated trajectory terminates at the correct product geometry. \textbf{Barrier Overshoot} measures the excess of the maximum predicted energy above the ground-truth barrier height, reflecting how accurately the energy barrier is reproduced. Finally, \textbf{Endpoint Energy Error} measures the absolute energy difference between the generated final state and the reference product, indicating whether the generated trajectory converges to the correct energetic minimum.

As shown in Table~\ref{tab:ablation_ws}, incorporating OS consistently improves all four metrics. In particular, the full model achieves approximately $1.3\times$ lower Path RMSD ($0.101$ vs.\ $0.129\,\mathrm{\AA}$), $4.5\times$ lower Endpoint RMSD ($0.013$ vs.\ $0.058\,\mathrm{\AA}$), $3.9\times$ lower Barrier Overshoot ($2.38$ vs.\ $9.26$ kcal/mol), and over $30\times$ lower Endpoint Energy Error ($0.055$ vs.\ $1.90$ kcal/mol). These improvements indicate that supervising the model with dense intermediate states allows it to better capture both the geometric evolution and the underlying energy landscape of the reaction pathway, rather than merely producing an accurate transition-state geometry.

\section{Reaction Path Initialization for NEB}
\label{app:neb}

We provide implementation details of the NEB initialization study reported in the main text. All experiments are performed on the Transition1x~\cite{schreiner2022transition1x} test split (287 reactions), using the DFT ($\omega$B97X/6-31G(d)) reactant, transition state (TS), and product geometries as reference.

\subsection{Experimental Setup}
\label{app:neb-procedure}

\paragraph{Initializations.}
We compare four initialization strategies: (i) \emph{Linear}, Cartesian interpolation between the reactant and product; (ii) \emph{IDPP}~\cite{smidstrup2014improved}; (iii) \emph{MEPIN}; and (iv) \emph{Ours}. Since MEPIN and Ours produce different numbers of frames, their trajectories are resampled to the target image count using arc-length interpolation. MEPIN uses a single sample (seed 0) for parity with the other methods.

\paragraph{Common NEB settings.}
All methods are optimized under identical settings with $9$ images (two fixed endpoints and seven interior images). NEB optimization is performed using the UMA interatomic potential (\texttt{uma-s-1p2}, \texttt{omol})~\cite{wood2026family} through ASE, together with the FIRE optimizer ($k=0.1~\mathrm{eV/\AA^2}$, $f_{\max}=0.1~\mathrm{eV/\AA}$, maximum $500$ optimization steps). DFT geometries are used only for evaluation.

\paragraph{Optimization cost.}
Optimization cost is measured by the number of single-point energy/force evaluations during NEB optimization, providing a hardware-independent measure of computational cost.

\subsection{Evaluation Protocol}
\label{app:validity}

\paragraph{Clean set.}
Since UMA is used during optimization, reactions whose UMA saddle does not coincide with the DFT TS cannot recover the DFT pathway under any initialization. We therefore relax each DFT TS on the UMA surface using Sella~\cite{hermes2019accelerated} and retain reactions whose optimized structure remains within $0.1~\mathrm{\AA}$ of the DFT TS, resulting in a clean set of $249/287$ reactions ($87\%$). All subsequent analyses are performed on this clean set.

\paragraph{Path validity.}
An optimized band is considered to recover the DFT reaction if (i) NEB converges, (ii) the energy profile forms a physical single barrier, and (iii) the highest-energy image matches the DFT TS with permutation-aware RMSD below $0.2~\mathrm{\AA}$ using the \texttt{pymatgen} molecule matcher. We additionally require a barrier height above $0.05~\mathrm{eV}$.

\paragraph{Efficiency comparison.}
Optimization cost is compared only on reactions for which all methods recover the DFT reaction (the same-destination intersection). We report median force evaluations and optimizer steps, using paired Wilcoxon signed-rank tests for optimization cost and McNemar tests for paired success rates.

\paragraph{Reaction difficulty.}
We additionally analyze the hardest $30\%$ of reactions according to either the DFT barrier height or the reactant--product Kabsch RMSD. Heavy-atom count is not used because larger molecules often produce artificially lower RMSD.

\paragraph{Convergence versus validity.}
Raw NEB convergence rates are similar across methods. However, convergence alone does not guarantee recovery of the correct reaction channel. We therefore report path validity rather than convergence rate as the primary evaluation criterion.

\section{Analysis on Off-Path Self-Correction Dynamics}

\subsection{Case Studies of Off-Path Self-Correction Dynamics}

To better understand how the proposed off-path correction affects inference, we visualize representative reaction trajectories together with the learned velocity fields projected onto the first two principal components. In each case, panel (a) shows the rollout of the baseline model without off-path correction, panel (b) its corresponding velocity field, panel (c) the velocity field learned with our correction mechanism, and panel (d) the rollout of our full model. Red arrows indicate velocity components that restore states toward the reference path, whereas blue arrows indicate components that push states away from it.

\begin{figure*}[t]
    \centering
    \includegraphics[width=\textwidth]{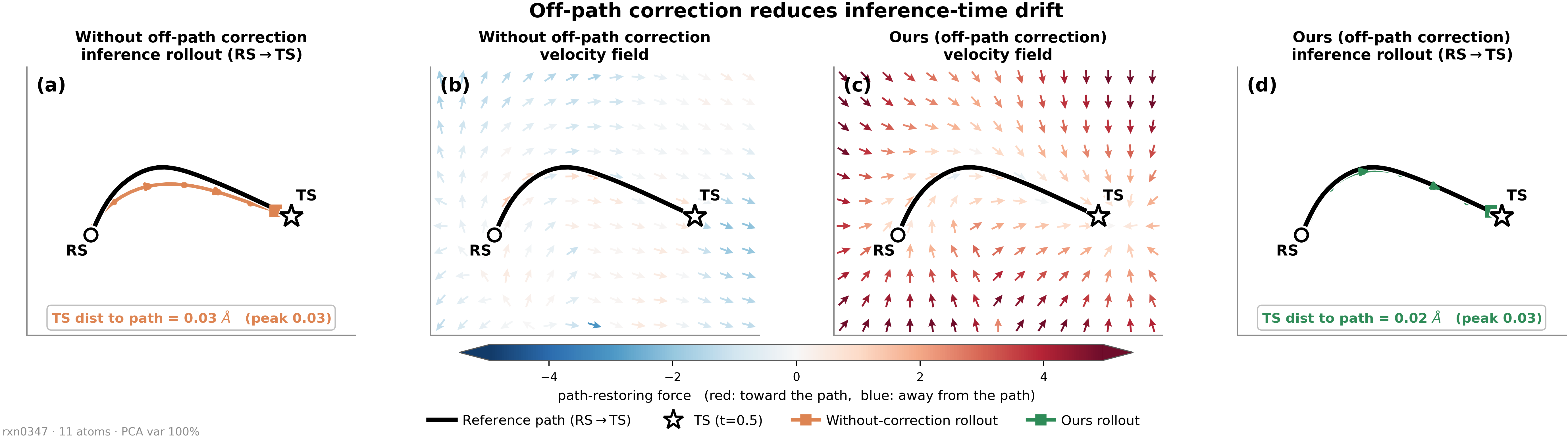}
    \caption{Case study on reaction rxn0347.}
    \label{fig:exposure_bias_0347}
\end{figure*}

\paragraph{Reaction rxn0347.}
Both models predict the TS accurately. However, our rollout follows the reference path more closely throughout integration, resulting in a smaller path deviation despite similar TS accuracy.

\begin{figure*}[t]
    \centering
    \includegraphics[width=\textwidth]{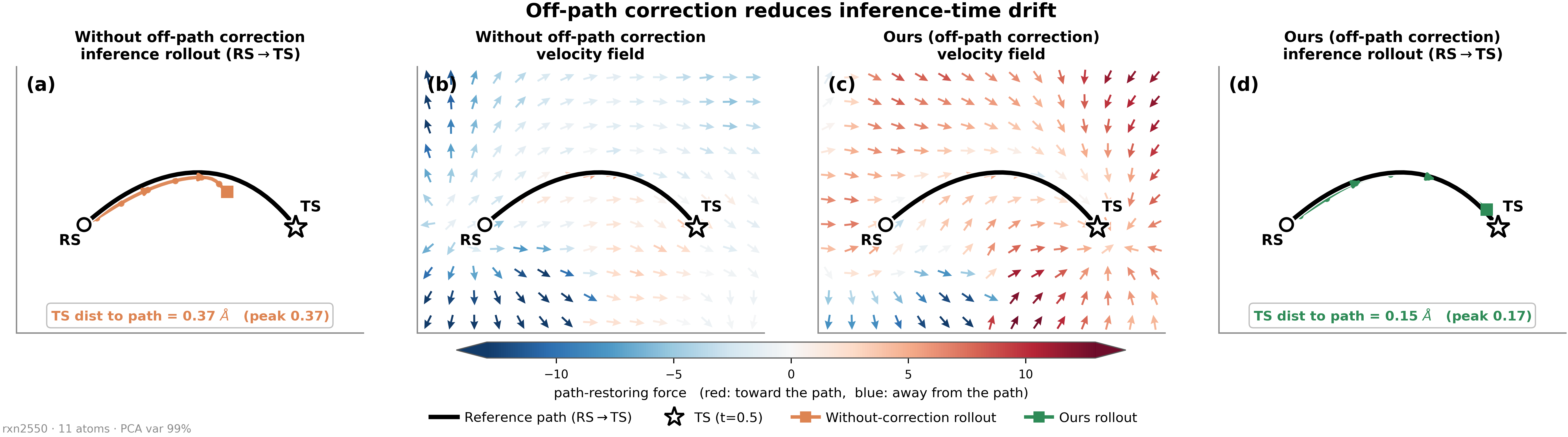}
    \caption{Case study on reaction rxn2550.}
    \label{fig:exposure_bias_2550}
\end{figure*}

\paragraph{Reaction rxn2550.}
The baseline gradually loses guidance toward the TS and deviates from the reference path. Our corrected velocity field provides a stronger restoring force, allowing the rollout to remain closer to the path and reach the TS more accurately.

\begin{figure*}[t]
    \centering
    \includegraphics[width=\textwidth]{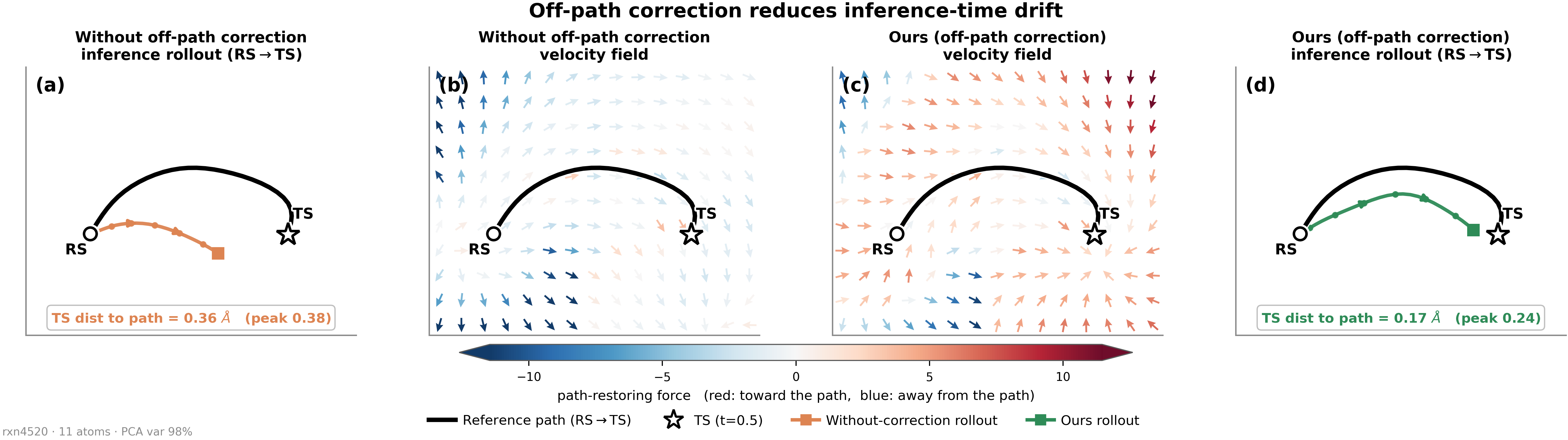}
    \caption{Case study on reaction rxn4520.}
    \label{fig:exposure_bias_4520}
\end{figure*}

\paragraph{Reaction rxn4520.}
This reaction contains a sharp turn immediately before the TS. Although our method reduces the deviation, it still struggles to follow this highly curved segment, indicating that paths with large local curvature remain challenging.

\begin{figure*}[t]
    \centering
    \includegraphics[width=\textwidth]{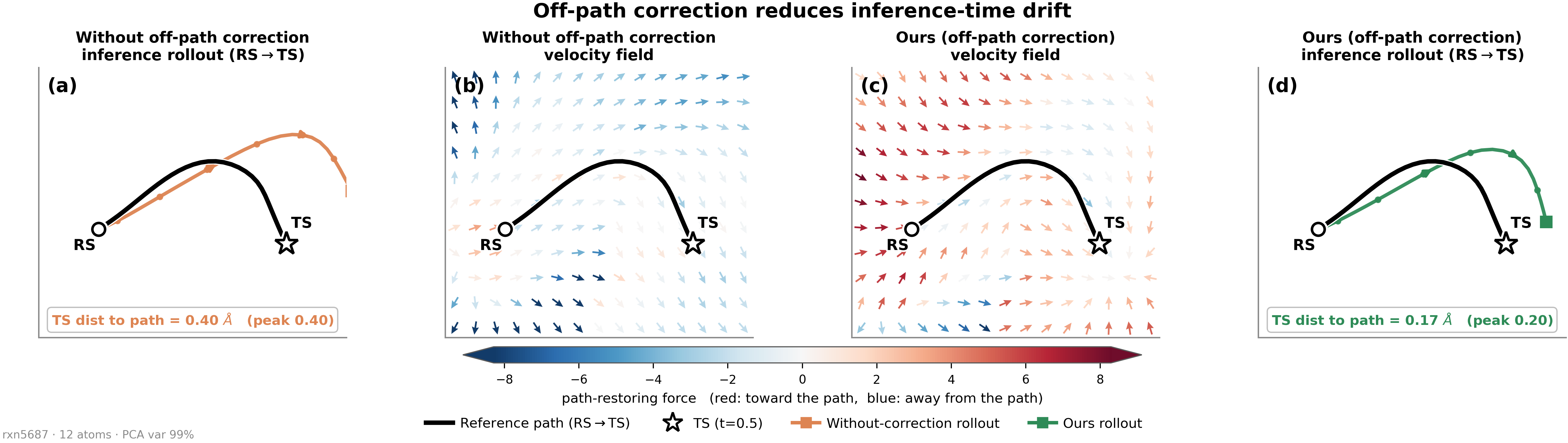}
    \caption{Case study on reaction rxn5687.}
    \label{fig:exposure_bias_5687}
\end{figure*}

\paragraph{Reaction rxn5687.}
The baseline diverges completely from the reference path, whereas our model successfully recovers the trajectory toward the TS. However, a small path-repelling region remains in the learned velocity field (panel (c)), leading to a residual deviation before reaching the TS.

\section{Kinetics-guided Reaction Design}

\begin{figure*}[t]
    \centering
    \begin{minipage}[t]{0.36\textwidth}
        \centering
        \includegraphics[width=\linewidth]{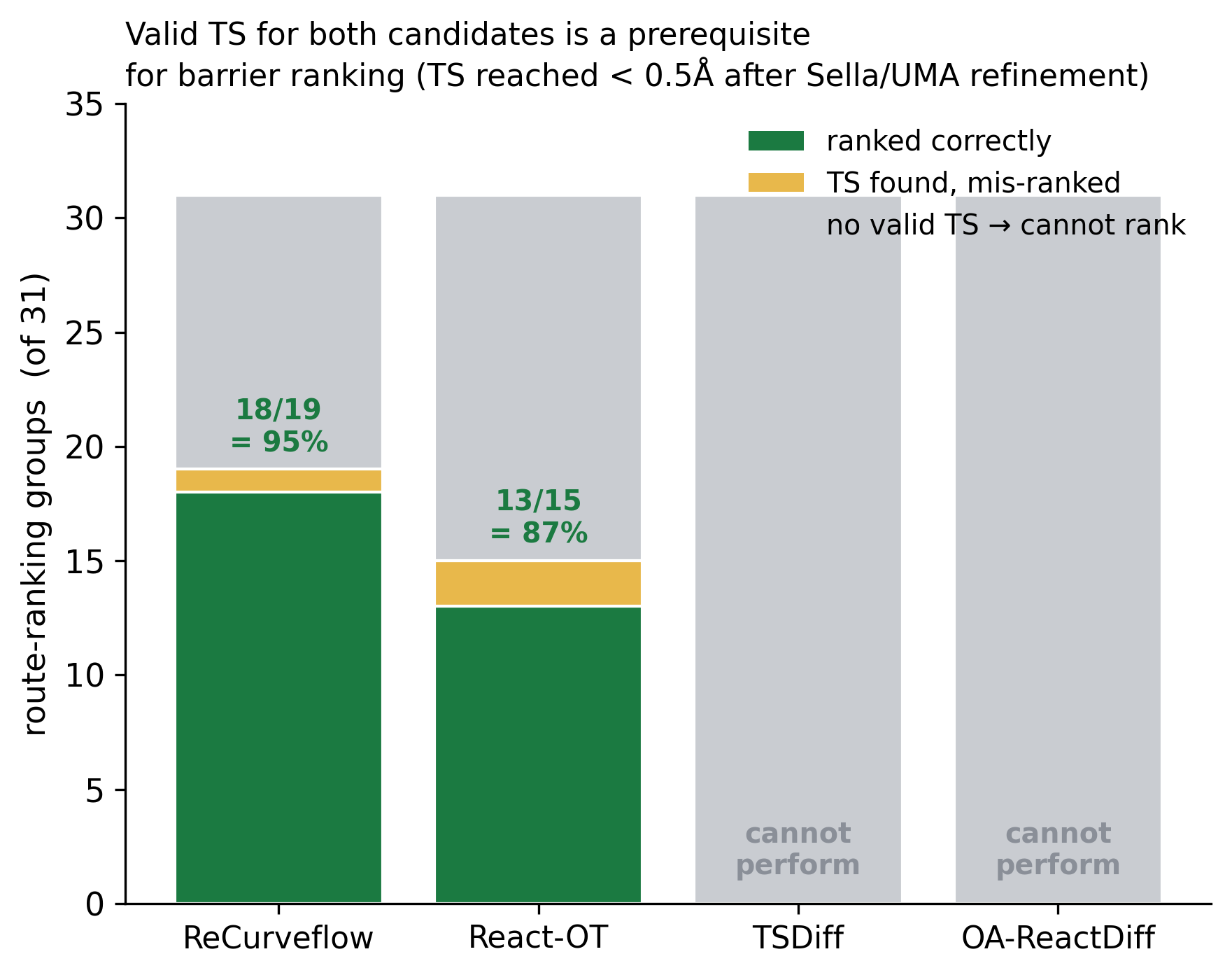}
        \subcaption{Capability for kinetics-guided reaction design.}
        \label{fig:capability}
    \end{minipage}
    \hfill
    \begin{minipage}[t]{0.60\textwidth}
        \centering
        \includegraphics[width=\linewidth]{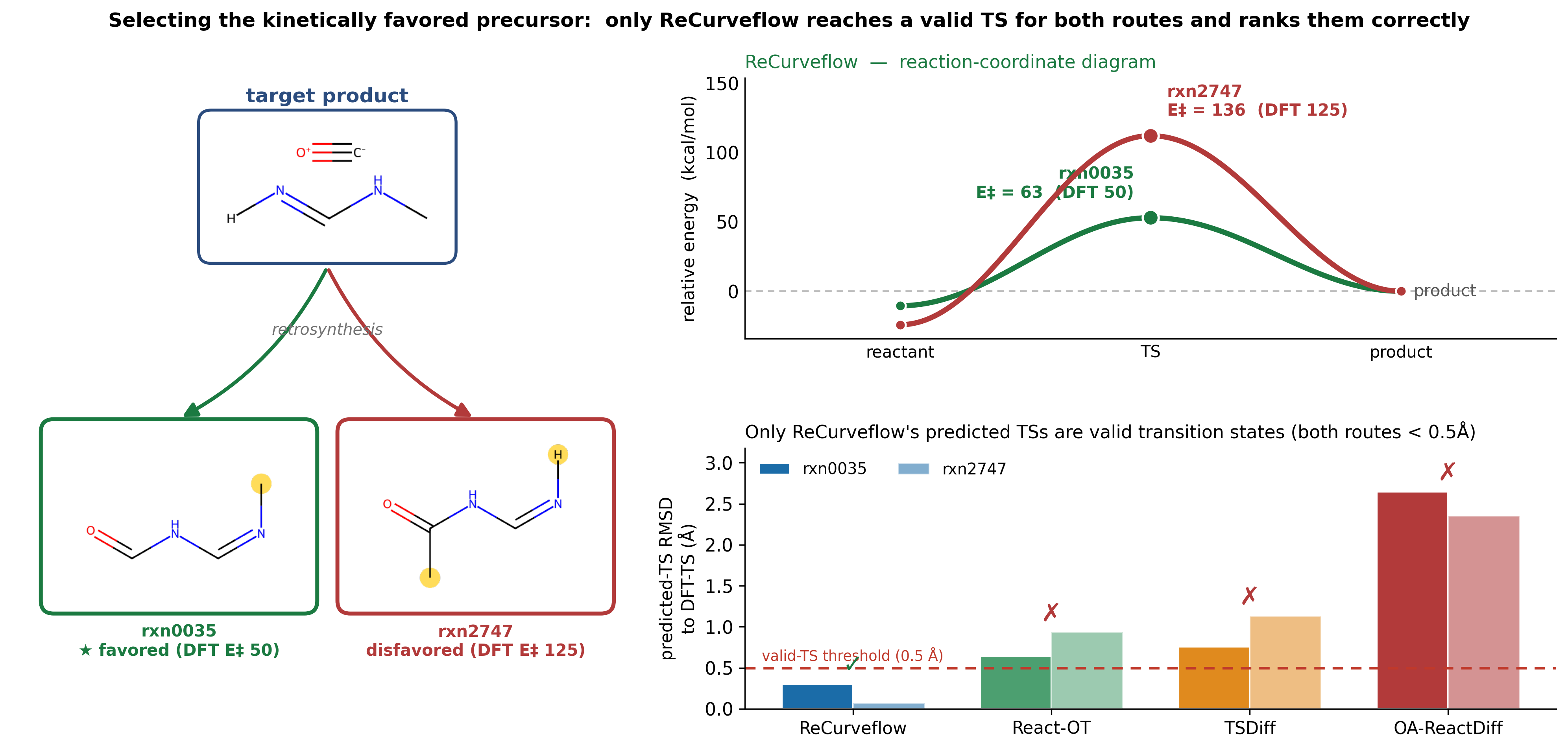}
        \subcaption{Representative reaction-group example.}
        \label{fig:casestudy}
    \end{minipage}

    \caption{
    Kinetics-guided reaction design.
    (a) Activation-barrier ranking is feasible only when valid transition-state geometries are obtained for both candidate reactions.
    ReCurveflow generates valid transition states substantially more often than previous methods, enabling reliable pathway ranking.
    (b) Representative reaction group producing the same product.
    ReCurveflow accurately predicts both transition-state geometries and correctly identifies the lower-barrier pathway, whereas the baselines fail to obtain a valid transition state for at least one reaction.
    }
    \label{fig:kinetics}
\end{figure*}

When multiple elementary reactions produce the same product, the kinetically preferred pathway is determined by the activation barrier.
A practical transition-state prediction model should therefore not only generate accurate transition-state geometries but also enable reliable activation-barrier ranking.
To evaluate this capability, we constructed 31 reaction groups from the unseen formula split of Transition1x, where each group contains reactions sharing the same product but different reactants.
The activation barrier was computed as
\[
E^\ddagger = E_{\mathrm{UMA}}(\mathrm{TS}) - E_{\mathrm{UMA}}(\mathrm{R}).
\]

Before evaluating predicted transition-state geometries, we first verified UMA as a surrogate evaluator by rescoring the reference DFT transition states, obtaining 94\% agreement with the DFT barrier ranking.
Since geometrically inaccurate transition states may occasionally produce plausible energies, we further optimized every predicted transition-state geometry using the same Sella/UMA saddle optimization procedure and regarded a prediction as valid only if the optimized structure converged to the reference DFT transition state.
Under this criterion, ReCurveflow reached the reference transition state in 78\% of reactions within 0.5\,\AA\ and in 100\% within 1.0\,\AA, compared with 68\% for React-OT, 10\% for TSDiff, and 0\% for OA-ReactDiff.

Activation-barrier ranking is possible only when valid transition-state geometries are obtained for both candidate reactions, making transition-state validity a prerequisite for practical pathway selection (Fig.~\ref{fig:capability}).
ReCurveflow successfully ranked 19 of the 31 reaction groups and correctly identified the lower-barrier pathway in 18 of them (95\%).
React-OT achieved 13 correct rankings among 15 feasible reaction groups (87\%), whereas TSDiff and OA-ReactDiff failed to produce valid transition-state geometries for any reaction group and therefore could not perform the task.
As shown in Fig.~\ref{fig:casestudy}, ReCurveflow accurately predicted the transition-state geometries of both \texttt{rxn0035} and \texttt{rxn2747}, correctly identifying the lower-barrier pathway, while all baselines failed to obtain a valid transition state for at least one reaction.
These results demonstrate that ReCurveflow is capable of supporting reliable \emph{kinetics-guided reaction design} through accurate transition-state prediction and activation-barrier ranking.


\end{appendices}
\newpage

\end{document}